\PassOptionsToPackage{unicode}{hyperref}
\PassOptionsToPackage{hyphens}{url}
\PassOptionsToPackage{dvipsnames,svgnames,x11names}{xcolor}
\documentclass[
  12pt]{article}

\usepackage{amsmath,amssymb}
\usepackage{iftex}
\ifPDFTeX
  \usepackage[T1]{fontenc}
  \usepackage[utf8]{inputenc}
  \usepackage{textcomp} 
\else 
  \usepackage{unicode-math}
  \defaultfontfeatures{Scale=MatchLowercase}
  \defaultfontfeatures[\rmfamily]{Ligatures=TeX,Scale=1}
\fi
\usepackage{lmodern}
\ifPDFTeX\else  
\fi
\IfFileExists{upquote.sty}{\usepackage{upquote}}{}
\IfFileExists{microtype.sty}{
  \usepackage[]{microtype}
  \UseMicrotypeSet[protrusion]{basicmath} 
}{}
\makeatletter
\@ifundefined{KOMAClassName}{
  \IfFileExists{parskip.sty}{%
    \usepackage{parskip}
  }{
    \setlength{\parindent}{0pt}
    \setlength{\parskip}{6pt plus 2pt minus 1pt}}
}{
  \KOMAoptions{parskip=half}}
\makeatother
\usepackage[table]{xcolor}
\makeatletter
\ifx\paragraph\undefined\else
  \let\oldparagraph\paragraph
  \renewcommand{\paragraph}{
    \@ifstar
      \xxxParagraphStar
      \xxxParagraphNoStar
  }
  \newcommand{\xxxParagraphStar}[1]{\oldparagraph*{#1}\mbox{}}
  \newcommand{\xxxParagraphNoStar}[1]{\oldparagraph{#1}\mbox{}}
\fi
\ifx\subparagraph\undefined\else
  \let\oldsubparagraph\subparagraph
  \renewcommand{\subparagraph}{
    \@ifstar
      \xxxSubParagraphStar
      \xxxSubParagraphNoStar
  }
  \newcommand{\xxxSubParagraphStar}[1]{\oldsubparagraph*{#1}\mbox{}}
  \newcommand{\xxxSubParagraphNoStar}[1]{\oldsubparagraph{#1}\mbox{}}
\fi
\makeatother

\usepackage{longtable,booktabs,array}
\usepackage{calc} 
\usepackage{etoolbox}
\makeatletter
\patchcmd\longtable{\par}{\if@noskipsec\mbox{}\fi\par}{}{}
\makeatother
\IfFileExists{footnotehyper.sty}{\usepackage{footnotehyper}}{\usepackage{footnote}}
\makesavenoteenv{longtable}
\usepackage{graphicx}
\makeatletter
\def\maxwidth{\ifdim\Gin@nat@width>\linewidth\linewidth\else\Gin@nat@width\fi}
\def\maxheight{\ifdim\Gin@nat@height>\textheight\textheight\else\Gin@nat@height\fi}
\makeatother
\setkeys{Gin}{width=\maxwidth,height=\maxheight,keepaspectratio}
\makeatletter
\def\fps@figure{htbp}
\makeatother

\makeatletter
\@ifpackageloaded{caption}{}{\usepackage{caption}}
\AtBeginDocument{%
\ifdefined\contentsname
  \renewcommand*\contentsname{Table of contents}
\else
  \newcommand\contentsname{Table of contents}
\fi
\ifdefined\listfigurename
  \renewcommand*\listfigurename{List of Figures}
\else
  \newcommand\listfigurename{List of Figures}
\fi
\ifdefined\listtablename
  \renewcommand*\listtablename{List of Tables}
\else
  \newcommand\listtablename{List of Tables}
\fi
\ifdefined\figurename
  \renewcommand*\figurename{Figure}
\else
  \newcommand\figurename{Figure}
\fi
\ifdefined\tablename
  \renewcommand*\tablename{Table}
\else
  \newcommand\tablename{Table}
\fi
}
\@ifpackageloaded{float}{}{\usepackage{float}}
\floatstyle{ruled}
\@ifundefined{c@chapter}{\newfloat{codelisting}{h}{lop}}{\newfloat{codelisting}{h}{lop}[chapter]}
\floatname{codelisting}{Listing}

\makeatother
\makeatletter
\@ifpackageloaded{caption}{}{\usepackage{caption}}
\@ifpackageloaded{subcaption}{}{\usepackage{subcaption}}
\makeatother

\ifLuaTeX
  \usepackage{selnolig}  
\fi
\usepackage{natbib}
\usepackage{bookmark}
\usepackage{booktabs}       
\usepackage{amsfonts}       
\usepackage{microtype}      
\usepackage[ruled,vlined,algosection]{algorithm2e}
\usepackage{subcaption}
\usepackage{caption}
\usepackage{graphicx}
\usepackage{amssymb}
\usepackage{amsmath}
\usepackage{float}

\newcommand{\reals}{{\mbox{\bf R}}}

\newcommand{\symm}{\mbox{\bf S}}  

\newcommand{\Tr}{\mathop{\bf Tr}}
\newcommand{\diag}{\mathop{\bf diag}}

\newcommand{\card}{\mathop{\bf card}}

\newcommand{\sign}{\mathop{\bf sign}}
\newcommand{\mnorm}[1]{{\vert\kern-0.3ex\vert\kern-0.3ex\vert #1 
\vert\kern-0.3ex\vert\kern-0.3ex\vert}}

\newcommand{\argmin}{\operatornamewithlimits{argmin}}

\newcommand{\R}{{\mathbf{R}}}

\newcommand{\E}{{\mathbb{E}}}

\newcommand{\corr}{\mathop{\bf Corr}}

\newcommand{\prob}{{\mathbb{P}}}

\newcommand{\prox}{\mathbf{Prox}}

\newcommand{\braket}[1]{\langle #1 \rangle}

\newcommand{\cf}{{\it cf.}}
\newcommand{\eg}{{\it e.g.}}
\newcommand{\ie}{{\it i.e.}}

{\catcode`p =12 \catcode`t =12 \gdef\eeaa#1pt{#1}}      
\def\accentadjtext#1{\setbox0\hbox{$#1$}\kern   
                 \expandafter\eeaa\the\fontdimen1\textfont1 \ht0 }
\def\accentadjscript#1{\setbox0\hbox{$#1$}\kern 
                 \expandafter\eeaa\the\fontdimen1\scriptfont1 \ht0 }
\def\accentadjscriptscript#1{\setbox0\hbox{$#1$}\kern   
                 \expandafter\eeaa\the\fontdimen1\scriptscriptfont1 \ht0 }
\def\accentadjtextback#1{\setbox0\hbox{$#1$}\kern       
                 -\expandafter\eeaa\the\fontdimen1\textfont1 \ht0 }
\def\accentadjscriptback#1{\setbox0\hbox{$#1$}\kern     
                 -\expandafter\eeaa\the\fontdimen1\scriptfont1 \ht0 }
\def\accentadjscriptscriptback#1{\setbox0\hbox{$#1$}\kern 
                 -\expandafter\eeaa\the\fontdimen1\scriptscriptfont1 \ht0 }

\makeatletter
\def\mdots@{\mathinner.\nonscript\!.%
  \ifx\next,.\else\ifx\next;.\else\ifx\next..\else
  \nonscript\!\mathinner.\fi\fi\fi}
\let\ldots\mdots@
\let\cdots\mdots@
\let\dotso\mdots@
\let\dotsb\mdots@
\let\dotsm\mdots@
\let\dotsc\mdots@
\def\vdots{\vbox{\baselineskip2.8\p@ \lineskiplimit\z@
     \kern6\p@\hbox{.}\hbox{.}\hbox{.}\kern3\p@}}
\def\ddots{\mathinner{\mkern1mu\raise8.6\p@\vbox{\kern7\p@\hbox{.}}%
     \raise5.8\p@\hbox{.}\raise3\p@\hbox{.}\mkern1mu}}
\makeatother

\RequirePackage{amsthm}
\newtheoremstyle{descriptive}%
	{\topsep}   
	{\topsep}   
	{\rmfamily} 
	{}          
	{\bfseries} 
	{.}         
	{ }         
	{}          
\newtheoremstyle{propositional}%
	{\topsep}   
	{\topsep}   
	{\itshape}  
	{}          
	{\bfseries} 
	{.}         
	{ }         
	{}          
\theoremstyle{propositional}
\newtheorem{thm}{Theorem}[section]

\theoremstyle{descriptive}

\newtheorem{assumption}{Assumption}[section]

\usepackage{array}
\newcolumntype{C}{>{\centering\arraybackslash}p{0.22\textwidth}}

\IfFileExists{xurl.sty}{\usepackage{xurl}}{} 
\hypersetup{
  pdftitle={main},
  pdfauthor={Samuel Erickson; Tobias Rydén},
  pdfkeywords={Precision matrix, high-dimensional, lasso},
  colorlinks=true,
  linkcolor={blue},
  filecolor={Maroon},
  citecolor={Blue},
  urlcolor={Blue},
  pdfcreator={LaTeX via pandoc}
}

\newcommand{\anon}{1}

\begin{document}

\def\spacingset#1{\renewcommand{\baselinestretch}%
	{#1}\small\normalsize} \spacingset{1}


\if1\anon
	{
		\title{
			\bf Non-asymptotic Estimation of Sparse Partial Correlation via Joint
			Partial Regression
		}
		\author{Samuel Erickson\thanks{
                The authors gratefully acknowledge the support of Lynx Asset
                Management during the early stages of this research. Part of
                this research was conducted while the first author was
                affiliated with Lynx Asset Management. The first author also
                gratefully acknowledges partial support from the Wallenberg AI,
                Autonomous Systems and Software Program (WASP) funded by the
                Knut and Alice Wallenberg Foundation.
        }\hspace{.2cm}\\
			Department of Decision and Control System, KTH EECS \\
			and \\
			Tobias Rydén \\
			Department of Statistics, Stockholm University\\
            Lynx Asset Management AB
        }
		\maketitle
	} \fi

\if0\anon
	{
		\bigskip
		\bigskip
		\bigskip
		\begin{center}
			{\LARGE\bf Non-asymptotic Estimation of Sparse Partial Correlation via Joint
				Partial Regression}
		\end{center}
		\medskip
	} \fi

\bigskip
\begin{abstract}
    Partial correlation analysis is widely used in modern multivariate
    analysis. However, prior methods lack non-asymptotic guarantees and have
    slow asymptotic estimation rates. We present a method for simultaneously
    estimating sparse high-dimensional partial correlation matrices and inverse
    covariance with non-asymptotic rates. The method is a two-stage estimation
    procedure which exploits the connection between the inverse covariance
    matrix and linear regression estimation method while enforcing positive
    semi-definiteness simultaneously. An efficient proximal splitting algorithm
    for numerically computing the estimate is also derived, and we provide a
    fast implementation of the algorithm. The effectiveness of the proposed
    method is demonstrated on both synthetic and real-world data.
\end{abstract}

\noindent%
{\it Keywords:} Precision matrix, Gaussian graphical models, high-dimensionality, lasso.
\vfill

\newpage
\spacingset{1.8} 

\section{Introduction}
Two important and closely related problems in statistical learning are the
problems of estimating a partial correlation network and the inverse covariance
matrix, also known as the precision matrix, from data. Partial correlation
networks, which generalize Gaussian graphical models, are used to model the
relationships between variables when controlling for other variables. Partial
correlation networks are used in a plethora of applications, \eg, the analysis
of gene expression data \citep{delaFuente04}, quantitative psychology
\citep{Epskamp18}, and the analysis of financial markets \citep{Kenett15}. The
precision matrix, from which we can obtain the partial correlation network, is
also of interest in its own right, as it appears in linear discriminant
analysis \citep{Hastie09} and in Markowitz portfolio selection
\citep{Markowitz52}. However, due to the high-dimensionality of the problem,
estimating a precision or partial correlation matrix is often challenging as
the number of parameters is of the order of the number of features squared. For
this reason, classical methods such as using the inverse of the sample
covariance matrix, are known to perform poorly whenever the number of
observation is not extremely large. Additionally, classical methods produce
estimates which are almost surely dense. This makes regularization crucial,
since in many applications we typically only have a moderate number of
observations, and since we typically seek a sparse estimate which gives rise to
a more parsimonious and interpretable network model.

\paragraph*{Related work.} Gaussian graphical models, and by extension sparse
precision matrix estimation, have been studied extensively in the literature,
with the \emph{covariance selection} problem originally being due to
\cite{Dempster72}. Today there are many methods for estimating the precision
matrix, which can be divided into three main categories: methods based upon
maximum likelihood estimation, methods based upon approximate inversion, and
methods based upon what we shall refer to as \emph{partial regression}. The
partial regression approach exploits a connection between the precision matrix
and linear regression, and is given its name due to its close relation to
partial correlations.

Perhaps the most popular method of estimation is $\ell_1$-regularized Gaussian
maximum likelihood estimation, commonly known as the graphical lasso. This
method was introduced independently by \cite{Yuan07} and \cite{Banerjee08}
(having slight variations), where the former work also provided a consistency
result in the high-dimensional setting. There has additionally been much work
on the problem of numerically solving the graphical lasso problem, such as the
work of \cite{Friedman08}, who also coined the name. A more recent addition to
methods in this vein is given by \cite{Bertsimas20}, which directly solves the
cardinality constrained maximum likelihood problem to optimality using
mixed-integer programming.

\cite{Cai11} provide a method of the second kind, \textsc{clime}, based on
minimizing the $\ell_1$-norm of the estimate under the constraint that the
estimate is sufficiently close to an inverse of the sample covariance matrix.
Robust extensions of this method are proposed by \cite{Zhao14} and
\cite{Wang17}, focusing on heavy-tailed and contaminated data, respectively.
The work of \cite{Cai16} provide an adaptive variant, \textsc{aclime}. They
also derive the minimax optimal rate in the $\ell_2$-operator norm, and show
that which their methods achieves it.

The present work is closely related to the works of \cite{Meinshausen06},
\cite{vandeGeer14}, and \cite{Yuan10}, who opt for estimation via partial
regression. \citeauthor{Meinshausen06} specifically consider the sparsity
pattern estimation problem in a high-dimensional setting, and use the lasso
\citep{Tibshirani96} to estimate the neighborhood of each node in the graphical
model. Their work provides consistency results of their proposed method, and
has been one of the most impactful works on the subject. The work of
\citeauthor{vandeGeer14} continues in the same vein and estimates the precision
matrix with the purpose of constructing confidence regions for high-dimensional
models. \citeauthor{Yuan10} provides a method that can be boiled down to
solving a sequence of linear programs, opting for the Dantzig selector
\citep{Candes07} in lieu of the lasso, and projecting the estimate onto the
space of symmetric matrices with respect to the operator $\ell_1$-norm. The
method is proven to  achieve the minimax optimal rate in the operator
$\ell_1$-norm over certain parameter spaces.

There are additionally methods that combine partial regression with
pseudo-likelihood to estimate the precision matrix. The first method in this
vein is the symmetric lasso procedure due to \cite{Friedman10}. Perhaps the
most recent contribution is given by \cite{Behdin23}, who pursue this approach
with discrete optimization, as they opt for direct penalization of the
cardinality, rather than its usual convex surrogate. \citeauthor{Behdin23} also
give a non-asymptotic error rate for their estimator.

Less work has been done on the problem of estimating high-dimensional partial
correlation matrices, despite their frequent use in data analysis.
\cite{Peng09} provide an alternating minimization method dubbed \textsc{space}
for estimating the partial correlation matrix via partial regression.
\cite{Khare15} also propose \textsc{concord}, a pseudo-likelihood partial
regression method which is similar to the symmetric lasso procedure.
\citeauthor{Peng09} and \citeauthor{Khare15} both provide asymptotic analysis
for their respective methods, achieving identical error rates.

\paragraph*{\bf Contributions.} The main contributions of the present work are
three-fold, and are as follows: 
\begin{itemize}

    \item We propose a method for simultaneously estimating a sparse partial
        correlation and precision matrix via linear regression, which we call
        \emph{joint partial regression}. For sub-Gaussian data, we provide
        non-asymptotic estimation error rates for both the partial correlation
        and precision matrices. To the best of our knowledge, our
        non-asymptotic result on partial correlation estimation is the first of
        its kind, and improves upon the prior rate by more than a factor
        $\sqrt{s}$, where $s$ is the number of non-zero partial correlations.

    \item We provide an efficient proximal algorithm for numerically fitting
        the proposed method which is well-suited for large-scale data sets. An
        implementation of the algorithm is also given in the form of a Python
        package written in Rust. 

    \item We conduct numerical studies on synthetic data, showing that the
        proposed method outperforms the benchmark methods uniformly over
        several precision matrix models and problem sizes. We also apply the
        method to real stock market data, showing that it is able to produce
        interpretable results on real data.

\end{itemize}

\paragraph*{\bf Notation.} We write $\symm^p$ for the set of real symmetric
$p\times p$ matrices, and for $\Omega\in\symm^p$ we write $\Omega\succeq 0$  if
$\Omega$ is positive semi-definite. If $\Omega$ and $\Omega'$ are both in
$\symm^p$ we write $\Omega\succeq\Omega'$ if $\Omega - \Omega' \succeq 0$. The
Frobenius norm is defined as $\|\Omega\|_{\rm F} = \sqrt{\Tr(\Omega^2)}$. The
cardinality function $\card$ returns the number of elements if given a set
argument, and the number of non-zero entries if given a vector argument. The
proximal operator of a closed, convex and proper function $f$ is denoted
$\prox_f$.

\section{The precision matrix and partial correlation}
\label{sec:precision-partial-correlation} 

Likely the most common use of the precision matrix is in the context of Gaussian
graphical models, where the connection between the precision matrix and the
conditional independence structure of jointly Gaussian random variables is
exploited. The foundation of this connection is in fact a link that the
precision matrix also has to linear regression. Suppose $Z$ is a
square-integrable mean-zero random vector, taking values in $\reals^p$, with
non-singular covariance matrix $\Sigma$ and inverse $\Omega = \Sigma^{-1}$.
Then the best linear unbiased predictor of a component $Z_j$ given all other
components
\[
    Z_{\setminus j} = (Z_1, \dots, Z_{j-1}, Z_{j+1}, \dots, Z_p)
\]
is $Z_{\setminus j}^\top \theta_j$, where
\[
    \theta_j = \Sigma_{\setminus j, \setminus j}^{-1} \Sigma_{\setminus j, j}.
\]
Here $\Sigma_{\setminus j, \setminus j}$ is the sub-matrix of $\Sigma$ obtained
by removing the $j$-th row and column, and $\Sigma_{\setminus j, j}$ is the $j$-th
column of $\Sigma$ with the $j$-th element removed. Likewise, the residual
$\varepsilon_j = Z_j - Z_{\setminus j}^\top \theta_j$ has mean zero and
variance $\tau_j^2 = \Sigma_{jj} - \Sigma_{j, \setminus j} \Sigma_{\setminus j,
\setminus j}^{-1} \Sigma_{\setminus j, j}$. Using Schur complements, we have
that
\[
    \begin{aligned}
      \Omega_{jj} &= (\Sigma_{jj} - \Sigma_{j, \setminus j} \Sigma_{\setminus j, \setminus j}^{-1} \Sigma_{\setminus j, j})^{-1}, \\
      \Omega_{\setminus j, j} &= -(\Sigma_{jj} - \Sigma_{j, \setminus j} \Sigma_{\setminus j, \setminus j}^{-1} \Sigma_{\setminus j, j})^{-1} \Sigma_{\setminus j, \setminus j}^{-1} \Sigma_{\setminus j, j},
    \end{aligned}
\]
and thus
\begin{equation}\label{eq:partial-regression}
    \begin{aligned}
      \Omega_{jj} &= 1 / \tau_j^2, \\
      \Omega_{\setminus j, j} &= -\theta_j / \tau_j^2. 
    \end{aligned}
\end{equation}
In the Gaussian case we find that $\varepsilon_j$ is independent of
$Z_{\setminus j}$. This then implies that $Z_j$ and $Z_k$ are
\emph{conditionally independent}, given $Z_{\setminus \{j, k\}}$ if and only if
$\theta_{jk} = 0$, \ie, $\Omega_{jk} = 0$. More generally, the \emph{partial
correlation coefficient} between $Z_j$ and $Z_k$ given $Z_{\setminus \{j, k\}}$
is given by the negative correlation between $\varepsilon_j$ and
$\varepsilon_k$, 
\[
    \rho_{jk \mid \setminus\{j, k\}} = -\corr(\varepsilon_j, \varepsilon_k) = -\frac{\E(\varepsilon_j \varepsilon_k)}{\sqrt{\E(\varepsilon_j^2) \E(\varepsilon_k^2)}},
\]
which can be written in terms of the precision matrix and the linear regression
coefficients as
\[
    \rho_{jk \mid \setminus\{j, k\}} = -\frac{\Omega_{jk}}{\sqrt{\Omega_{jj} \Omega_{kk}}} = \frac{\tau_k}{\tau_j} \theta_{jk},
\]
respectively. Thus we define the \emph{partial correlation matrix} $Q$ as
\[
    Q = - T \Omega T,
\]
where $T = \diag(\tau_1, \dots, \tau_p)$. (Note that some works \citep{Schäfer04,
Marrelec06} define the partial correlation matrix as $-T \Omega T + 2I$, so that
the diagonal elements are $+1$ rather than $-1$.)

\section{Proposed method}
\label{sec:method}

In light of the connection between the precision matrix and linear regression
described in \S\ref{sec:precision-partial-correlation}, a way to estimate the
precision matrix is to form estimates of the linear regression coefficients
$\hat\theta_j$ and the residual variances $\hat \tau_j^2$ for each $j$ and then
compute an estimate of the precision matrix $\widehat{\Omega}$ via 
\[
    \begin{aligned}
        \widehat{\Omega}_{jj} &= 1 / \hat\tau_j^2, \\
        \widehat{\Omega}_{\setminus j, j} &= -\hat\theta_j / \hat\tau_j^2,
    \end{aligned}
\]
and an estimate of the partial correlation matrix $\widehat{Q}$ via
\[
    \widehat{Q} = -\widehat{T} \widehat{\Omega} \widehat{T},
\]
where $\widehat{T} = \diag(\hat\tau_1, \dots, \hat\tau_p)$. An obvious approach,
since we are interested in sparsity, is to use the lasso. Suppose we have access
to a data matrix $X \in \reals^{n \times p}$, where each row contains a sample
of the random vector $Z$ and each column corresponds to a feature of the data.
We can then estimate the linear regression coefficients by solving the lasso
problem
\[
    \hat \theta_j = \argmin_{\theta \in \mathbf{R}^{p-1}} \left\{\frac{1}{2n} \|X_j - X_{\setminus j} \theta\|_2^2 + \lambda \|\theta\|_1 \right\}
\]
and subsequently estimate the residual variances by $\hat \tau_j^2 = (1/n) \|X_j
- X_{\setminus j} \hat \theta_j\|_2^2$. However, this approach does not result
in a symmetric estimate of the precision matrix, let alone a positive
semi-definite one. If we are solely interested in the unweighted partial
correlation network, we could choose to include the edge $(j,k)$ only if
$\widehat{\Omega}_{jk}$ and $\widehat{\Omega}_{kj}$ are both non-zero or,
alternatively, if either is non-zero. If we are interested in the partial
correlations themselves we could choose to symmetrize the estimate by some form
of post-processing, \eg, averaging or projecting onto the semi-definite cone.
We will however take a different approach, in which we enforce positive
semi-definiteness while simultaneously estimating the linear regression
coefficients, given initial estimates of the residual variances. 

\subsection{Joint partial regression}
We propose the following convex program that regresses the features at the same
time as ensuring positive semi-definiteness,
\begin{equation}\label{eq:cvx-prob}
    \begin{array}{ll}
        \text{minimize}   & \displaystyle \sum_{j=1}^{p} \left( \frac{1}{2n} \|X_j - X_{\setminus j} \theta_j\|_2^2 + \lambda \|\theta_j\|_1 \right)\\
        \text{subject to} & \Omega_{jj} = 1 / \hat\tau_j^2, \quad \Omega_{\setminus j, j} = -\theta_j / \hat\tau_j^2, \\
                          & \Omega \succeq 0, \quad Q = -\widehat{T} \Omega \widehat{T}.
    \end{array}
\end{equation}
Here $\Omega\in\reals^{p\times p}$ and $\theta_j\in\reals^{p-1}$, $j=1,2,\dots,p$, are
the optimization variables, and the data are the matrix $X$, the estimates
$\hat\tau_j^2$ of the residual variances, and the regularization parameter
$\lambda \geq 0$. Note that the equality $Q = -\widehat{T}\Omega\widehat{T}$ is
not an actual constraint as $Q$ does not appear in the objective function or other
constraints, but simply expresses that $Q$ is a by-product of the optimization
problem in the stated way. We call the solutions $\widehat{\Omega}$ and
$\widehat{Q}$ the \emph{joint partial regression} estimates of the precision
matrix and the partial correlation matrix, respectively. In order to obtain the
initial estimates of the residual variances we can use the aforementioned lasso
estimates of the linear regression coefficients to compute the sample variances
$\hat \tau_j^2 = (1/n) \|X_j - X_{\setminus j} \hat \theta_j\|_2^2$. We
summarize the proposed procedure in Algorithm~\ref{algo:jpr}.

There are several benefits to constraining the regression to guarantee positive
semi-definiteness, rather than relying on some form of post-processing. First,
projecting onto the positive semi-definite cone does not preserve the sparsity
pattern of the estimate. It is also possible to project while constraining the
sparsity pattern, but by doing this we lose theoretical guarantees. Second, the
constraints provides free implicit regularization in the regression.

Note moreover that applying the derivation in
\S\ref{sec:precision-partial-correlation} to the sample covariance matrix
reveals that Algorithm~\ref{algo:jpr} exactly retrieves the inverse sample
covariance matrix, if it exists, when $\lambda = 0$.

\begin{algorithm}[ht]
    \SetAlgoLined 
    
    \textbf{Input:} {Data matrix $X \in \reals^{n \times p}$, penalty parameter
    $\lambda$.} \\

    \caption{\sc Joint partial regression}
    
    \BlankLine
    
    \textbf{for} {$j = 1, \dots, p$} \textbf{do} {
    \begin{quote} \it lasso regression
        \[
            \hat\theta_j = \argmin_{\theta \in \R^{p-1}} \left\{\frac{1}{2n} \|X_j - X_{\setminus j} \theta\|_2^2 + \lambda \|\theta\|_1\right\},
        \]
        and compute the estimate 
        \[
            \hat\tau_j^2 = (1/n) \|X_j - X_{\setminus j}\hat\theta_j\|_2^2
        \] 
        of the residual variance.
    \end{quote}
    } \textbf{end}
    
    \textit{Solve the convex program \eqref{eq:cvx-prob} with the estimated residual
    variances $\hat\tau_j^2$ and regularization parameter $\lambda$ to obtain the
    estimates $\widehat{\Omega}$ and $\widehat{Q}$ of the precision matrix and
    partial correlation matrix, respectively.}
    \label{algo:jpr}
\end{algorithm}

\subsection{Proximal splitting algorithm}\label{ssec:algorithm}

Consider a generalization to problem \eqref{eq:cvx-prob}, where we allow for
any convex and $L$-smooth loss $\ell$, a different regularization parameter
$\lambda_j$ for each feature $j$, and eigenvalue limits $\alpha$ and $\beta$:
\begin{equation}\label{eq:cvx-prob-general}
    \begin{array}{ll}
        \text{minimize} & \displaystyle f(\Omega) + g(\Omega) \\
        \text{subject to} & \alpha I \preceq \Omega \preceq \beta I
    \end{array}
\end{equation}
where
\[
    f(\Omega) = \sum_{j=1}^{p} \ell(X_j + \hat\tau_j^2X_{\setminus j} \Omega_{\setminus j , j})
\]
is the sum of partial regression losses, and
\[ 
    g(\Omega) = \sum_{j=1}^{p} \left(\lambda_j \hat\tau_j^2 \|\Omega_{\setminus j, j}\|_1 + \delta_{\{1/\hat\tau_j^2\}}(\Omega_{jj}) \right), \quad \delta_{\{1/\hat\tau_j^2\}}(\Omega_{jj}) = \begin{cases}
        0      & \text{if } \Omega_{jj} = 1/\hat\tau_j^2 \\
        \infty & \text{otherwise},
    \end{cases}
\]
is the sum of $\ell_1$-regularization terms and indicator functions on the
diagonal elements. The generalized joint partial regression problem
\eqref{eq:cvx-prob-general} can be solved via the \textsc{pd3o} algorithm
\citep{Yan18}. Introducing primal and dual iterate sequences $\Omega^{(k)}$ and
$U^{(k)}$, respectively, the \textsc{pd3o} algorithm in this instances has
iteration updates 
\[
    \begin{aligned}
        \Omega^{(k+1)} &= \Pi_{\mathcal{S}} \left(\Omega^{(k)} - \gamma U^{(k)} - \gamma \nabla f(\Omega^{(k)}) \right), \\
            U^{(k+1)} &= \prox_{\eta g^*}\Big(U^{(k)} + \eta (2 \Omega^{(k+1)} - \Omega^{(k)}) + \gamma \eta (\nabla f(\Omega^{(k)}) - \nabla f(\Omega^{(k+1)})) \Big).
    \end{aligned}
\]
where $\Pi_{\mathcal{S}}$ is the Euclidean projection onto the set $\mathcal{S} =
\{\Omega \in \symm^p \colon \alpha I \preceq \Omega \preceq \beta I\}$, and
$g^*$ is the Fenchel conjugate of $g$. The iterates $\Omega^{(k)}$ and
$U^{(k)}$ are guaranteed to converge to a primal and a dual solution of the
problem, respectively, if $\gamma$ and $\eta$ are chosen as positive step sizes
satisfying $\gamma < 2 / L$ and $\gamma \eta \leq 1$. 

\paragraph*{\bf Primal update.} The projection of a matrix $A \in \reals^{n \times n}$ onto 
$\mathcal{S}$ is by definition given by
\[
    \Pi_{\mathcal{S}}(A) = \argmin_{\Omega \in \mathcal{S}} \|\Omega - A\|_{\rm F}^2.
\]
Importantly, note that the matrix $A$ which is being projected need not be
symmetric in general. However, due to the symmetry of $\Omega \in \mathcal{S}$
the squared distance $\|\Omega - A\|_{\rm F}^2$ can be written as 
\begin{equation}\label{eq:projection}
    \sum_{j=1}^{p}(\Omega_{jj} - A_{jj})^2  + 2\sum_{j=1}^{p}\sum_{k<j}^{p} (\Omega_{jk} - (A_{jk} + A_{kj}) / 2)^2 = \|\Omega - (A + A^\top) / 2\|_{\rm F}^2
\end{equation}
up to a term constant in $\Omega$. Since the Frobenius norm is invariant under
orthogonal transformations, if we let $Q \Lambda Q^\top = (A+ A^\top) / 2$ be
the eigenvalue decomposition of the symmetric part of $A$, then
\eqref{eq:projection} can equivalently be written as
\[
    \|Q^\top(\Omega - (A + A^\top) / 2)Q\|_{\rm F}^2 = \|Q^\top\Omega Q - \Lambda\|_{\rm F}^2.
\]
Clearly, this expression is minimized by choosing $\Omega$ such that $Q^\top
\Omega Q$ is diagonal, with the diagonal elements being the projection of the
eigenvalues of ($A + A^\top) / 2$ onto the interval $[\alpha, \beta]$. Thus the
projection $\Pi_{\mathcal{S}}(A)$ is given by projecting the symmetric part of
$A$, 
\begin{align*}
    \Pi_{\mathcal{S}}(A) &= Q \Lambda_{[\alpha, \beta]} Q^\top
\end{align*}
where $\Lambda_{[\alpha, \beta]}$ is a diagonal matrix with entries
$(\Lambda_{[\alpha, \beta]})_{jj} = \min(\max(\Lambda_{jj}, \alpha), \beta)$. 

%
The gradient that appears in the primal step is given by
\[
    \begin{aligned}
        \nabla f(\Omega)_{jj} &= 0 \\
        \nabla f(\Omega)_{\setminus j, j} &= \hat \tau_j^2 X_{\setminus j}^\top \nabla \ell(X_{j} + \hat\tau_j^2 X_{\setminus j} \Omega_{\setminus j, j}),
    \end{aligned}
\]
and depends on the loss function we choose. If we take the quadratic loss as in
\eqref{eq:cvx-prob}, then for $i=1, \dots, n$, 
\[
    \nabla \ell(z)_i = \frac{z_i}{n}.
\]
For, \eg, the Huber loss
\[
    \ell(z) = \frac{1}{n} \sum_{i=1}^{n} \phi_\rho (z_i), \quad \phi_\rho(r) =
    \begin{cases} 
        r^2 / 2, & \text{if } |r| \leq \rho, \\
        \rho (|r| - \rho / 2), & \text{if } |r| > \rho,
    \end{cases}
\]
we have instead 
\[
    \nabla \ell(z)_i = \frac{1}{n} \begin{cases}
        z_i, & \text{if } |z_i| \leq \rho,\\
        \rho \sign(z_i), & \text{otherwise}.
    \end{cases}
\]
For two points $\Omega$ and $\Omega'$ in $\reals^{n\times n}$, with both the
quadratic or Huber loss, $f$ satisfies
\[
    \begin{aligned}
        \| \nabla f(\Omega) - \nabla f(\Omega')\|_{\rm F} &\leq \frac{1}{n} \left(\sum_{j=1}^{p} \|\hat\tau_j^4 X_{\setminus j}^\top X_{\setminus j}(\Omega_{\setminus j, j} - \Omega'_{\setminus j, j})\|_2^2\right)^{1/2} \\
                                                          &\leq \frac{1}{n} \max_{1\leq j\leq p} \hat\tau_j^4 \lambda_\text{max}(X_{\setminus j}^\top X_{\setminus j}) \|\Omega - \Omega'\|_{\rm F}. 
    \end{aligned}
\]
So $f$ is $L$-smooth with 
\[
    L = \max_{1\leq j\leq p} \hat\tau_j^4 \lambda_\text{max}(X_{\setminus j}^\top X_{\setminus j}) / n.
\]

\paragraph*{\bf Dual update.} By the Moreau identity \citep{Ryu22}, we can write the
dual iterate update with the proximal operator of $g / \eta$ instead of the
Fenchel conjugate $\eta g^*$,
\begin{align*}
    V^{(k+1)} &= U^{(k)} + \eta (2 \Omega^{(k+1)} - \Omega^{(k)}) + \gamma\eta(\nabla f(\Omega^{(k)}) - \nabla f(\Omega^{(k+1)})), \\
    U^{(k+1)} &= V^{(k+1)} - \eta \prox_{g / \eta}(V^{(k+1)} / \eta).
\end{align*}
The proximal operator of $g / \eta$ evaluated at $V / \eta$ is given by 
\[
\begin{aligned}
    \prox_{g / \eta}(V / \eta)_{jj} &= 1 / \hat\tau_j^2, \\
    \prox_{g / \eta}(V / \eta)_{kj} &= \sign(V_{kj})(V_{kj} - \hat\tau_j^2 \lambda_j)_+ / \eta,
\end{aligned}
\]
for $j = 1, \dots, p$ and $k \neq j$.

\paragraph*{\bf Termination criterion.} To stop the algorithm within a finite number
of steps, we terminate when the iterates are approximately stationary. That is,
given a tolerance $\epsilon^\text{tol} > 0$, the termination criterion is
\[
    \max \left(\|\Omega^{(k+1)} - \Omega^{(k)}\|_{\rm F}, \|U^{(k+1)} - U^{(k)}\|_{\rm F} \right) \leq \epsilon^\text{tol}.
\]

\paragraph*{\bf Computational complexity.} If $\alpha, \beta \in (-\infty, \infty)$
then the computational cost of the projection in the primal update is that of
an eigenvalue decomposition, which is $O(p^3)$. If we choose not to constrain
the eigenvalues, \ie, $\alpha = -\infty$ and $\beta = \infty$, this step can be
skipped. In the case of the quadratic loss the gradient computation is $O(p^3)$
if we initially cache the matrix products, and $O(p^2n)$ otherwise. The
computational cost of the dual update is $O(p^2)$, meaning the total per
iteration cost of the algorithm is $O(p^3)$ or $O(\min(p^3, p^2n))$ if the
eigenvalues are unconstrained. Comparing with other method, the graphical lasso
(which does constrain the eigenvalues) has $O(p^3)$ iteration complexity, while
\textsc{space} and \textsc{concord} (which do not control the eigenvalues) have
the reduced iteration complexity of $O(\min(p^3, p^2n))$. However, note with
these complexities that in practice the sparsity of the estimate will affect
the computation time significantly, meaning the data matrix computations may be
more costly than the eigenvalue decompositions. 

\paragraph*{\bf Practical considerations.} The initial lasso regressions in
Algorithm~\ref{algo:jpr} may give us several beneficial by-products. First,
forming a matrix via \eqref{eq:partial-regression} makes for an appropriate
initialization $\Omega^{(0)}$. Second, we may use cross-validation or
information criteria, such as the Akaike information criterion 
\[
    n \log \left(  \frac{1}{n} \|X_j - X_{\setminus j} \theta\|_2^2 \right) + 2 \card(\theta),
\]
or the Bayesian information criterion
\[
    n \log \left(  \frac{1}{n} \|X_j - X_{\setminus j} \theta\|_2^2 \right) + \card(\theta) \log(n),
\]
to select the regularization parameters before solving the optimization problem
\eqref{eq:cvx-prob-general}. In this way, we only need to run $p$ lasso
parameter selection procedures to select a parameter for each feature. If
instead done at the final estimate level, this would be prohibitively
expensive. Moreover, another benefit of the partial regression approach is the
we can easily swap out the quadratic loss for another, \eg, the Huber loss in
the case of contaminated or heavy-tailed data.

\section{Theoretical results}\label{sec:theory}

In this section we provide theoretical guarantees for the estimation error of
the proposed method that hold with high probability under given assumptions.
Each data point, a random vector in $\reals^p$, is assumed to have a true,
unknown, covariance matrix that we denote by $\Sigma^\star$. The corresponding
true precision and partial correlation matrix is denoted $\Omega^\star$ and
$Q^\star$, respectively. Our first assumption concerns how the dimensionality
of the problem grows. It specifies how fast $p$ may grow with $n$, and also
bounds the growth of the degree of the true partial correlation network.

\begin{assumption}\label{assumption:dimensionality}    
The dimensionality is such that $p/n \leq 1 - \delta$ for some $\delta \in (0,
1)$, and the degree $d = \max_j \sum_{k\neq j} \mathbf{1}(\Omega_{jk}^\star
\neq 0)$ of the partial correlation network is such that $d\sqrt{\log
(p)/n}\leq M$ for some constant $M>0$.  
\end{assumption}

Note that the dependence on $n$ and $p$ is in many cases not made explicit;
\eg, in the above assumption both $\Omega^\star$ and $d$ tacitly depend on $n$
and $p$. We remark moreover that the term $d\sqrt{\log(p)/n}$ that appears in
the assumption is interesting in its own right. It is the minimax error rate in
the operator $\ell_q$-norm for $1 \leq q \leq \infty$ as shown by \cite{Cai16},
and it is the error rate of the graphical lasso in the operator $\ell_2$-norm
(the spectral norm), as shown by \cite{Ravikumar11}. Noting that the spectral
norm is smaller than the Frobenius norm, we argue that the boundedness
assumption is comparatively weak.

Our second assumption is that the data is \emph{sub-Gaussian}. A real-valued
random variable $Y$ is said to be sub-Gaussian if there exists a constant $C$
such that the tail probabilities satisfy 
\[
    \prob(|Y| \geq t) \leq 2 \exp(-t^2 / C^2), \quad \text{ for all } t \geq 0,
\] 
and a multivariate random variable $Z$ is said to be sub-Gaussian if every
linear combination $\braket{z, Z}$ with $z \in \reals^p$ is sub-Gaussian. We
define the \emph{sub-Gaussian norm} $\|\cdot\|_{\psi_2}$ for real-valued random
variables by
\[
    \|Y\|_{\psi_2} = \inf\{t > 0 \colon \E(\exp(Y^2 / t^2)) \leq 2\},
\]
and for multivariate random variables by $\|Z\|_{\psi_2} = \sup_{\|z\|_2 = 1}
\|\braket{z, Z}\|_{\psi_2}$. 


\begin{assumption}\label{assumption:sub-gaussianity}
The rows $x_i$ of the data matrix $X \in \reals^{n \times p}$ are i.i.d.\
samples drawn from a sub-Gaussian distribution with covariance matrix
$\Sigma^\star$ and sub-Gaussian norm $\|x_{i}\|_{\psi_2} \leq K$ for some
constant $K > 0$. 
\end{assumption}

It is clear from Assumption~\ref{assumption:sub-gaussianity} that each element
$X_{ij}$ is sub-Gaussian with $\|X_{ij}\|_{\psi_2} \leq K$. We remark however
that the assumption that the row-wise norms $\|x_i\|_{\psi_2}$ stay bounded as
$p$ grows does not imply that the element-wise norms $\|X_{ij}\|_{\psi_2}$ need
to go to zero. This is because the multivariate sub-Gaussian norm measures,
loosely speaking, the size of the largest eigenvalue of the covariance matrix
\cite[\cf][for the Gaussian case]{Vershynin18}. Thus, as long as this
eigenvalue stays bounded, which is part of our next assumption, there is
generally no such tacit assumption on the $X_{ij}$. Moreover, we remark that if
we make the stronger assumption of Gaussianity we no longer require $d
\sqrt{\log(p) / n}$ to be bounded.

Finally we assume that the eigenvalues of the true covariance matrix do not
degenerate to zero or infinity.

\begin{assumption}\label{assumption:bounded-eigenvalues}    
The eigenvalues of the precision matrix $\Omega^\star = (\Sigma^\star)^{-1}$
are uniformly bounded away from $0$ and $\infty$, \ie, there exists a constant
$\kappa \in (1,\infty)$ such that $1/\kappa \leq
\lambda_\mathrm{min}(\Omega^\star) \leq \lambda_\mathrm{max}(\Omega^\star) \leq
\kappa.$
\end{assumption}

Defining the size $s$ of the partial correlation network as 
\[
    s = \card\{(j,k) \colon \Omega^\star_{jk} \neq 0,\ j > k\},
\]
we can now state the main result of this section under the assumptions above.

\begin{thm}\label{thm:est-rate}
    Under Assumptions \ref{assumption:dimensionality},
    \ref{assumption:sub-gaussianity} and \ref{assumption:bounded-eigenvalues},
    there exist positive constants $c$, $C_1$ and $C_2$ such that
    Algorithm~\ref{algo:jpr} with $\lambda = c\sqrt{\log(p)/n}$ outputs an
    estimate $\widehat{Q}$ of the partial correlation matrix that satisfies
    \begin{equation}\label{eq:Q-estimation-error}
        \|\widehat{Q} - Q^\star\|_{\rm F} \leq C_1 \sqrt{\frac{s\log p}{n}}    
    \end{equation}
    and an estimate $\widehat{\Omega}$ of the precision matrix that satisfies
    \begin{equation}\label{eq:Omega-estimation-error}
        \|\widehat{\Omega} - \Omega^\star\|_{\rm F} \leq C_2 \sqrt{\frac{(s+p)\log p}{n}}
    \end{equation}
    with probability at least $1 - 6/p$.
\end{thm}

The statistical rate of convergence \eqref{eq:Q-estimation-error} for the
partial correlation matrix estimate improves upon the asymptotic rate of
\textsc{space} by more than a factor of $\sqrt{s}$, since its rate is
proportional to $\sqrt{s} \lambda$ with the condition $\sqrt{s\log(n) / n} =
o(\lambda)$. Critically, the proof of \textsc{space}'s rate requires the
estimate to have no false positives in the sparsity pattern for the rate to
hold, which one should not expect in practice in a finite sample regime.

The rate \eqref{eq:Omega-estimation-error} for the precision matrix estimate
matches that of the graphical lasso proved by \cite{Rothman08}. Their result
however differs slightly from ours in the assumptions. Their distribution
assumption is stronger, relying on Gaussianity rather than sub-Gaussianity, but
their dimensionality assumption is weaker, allowing $p \geq n$. However, do
note that no estimate of the precision matrix is consistent in the Frobenius
norm if $p \geq n$ due to the diagonal. 




%

\section{Numerical experiments}\label{sec:experiments}

We study the performance of the proposed method and compare it to existing
estimation methods on synthetic data, and test the proposed method on real
stock market data. In \S\ref{ssec:synthetic-data} we evaluate the performance
of the estimators on synthetic data drawn from a known distribution, and
evaluate the estimation error and the $F_1$-score of the estimated networks. In
\S\ref{ssec:stock-market-data} we estimate the partial correlation network of
stock returns on the Nasdaq Stock Market. The experiments in
\S\ref{ssec:synthetic-data} were run on a cluster with 32 Intel Xeon processors
at 3.10 GHz, whereas the ones in \S\ref{ssec:stock-market-data} were run using
an Intel Core i9 processor at 3.20 GHz. 

\subsection{Synthetic data}\label{ssec:synthetic-data}
We begin by evaluating the performance of the joint partial regression method on
synthetic data for which the distribution and the true precision matrix are known.
We compare the performance of the proposed method with the graphical lasso,
\textsc{space} and \textsc{concord}. 

\paragraph*{\bf Data generation.} For different numbers of features $p$ we generate
data sets of $n=500$ samples from a $p$-dimensional Gaussian distribution
$\mathcal{N}(0, \Sigma)$. We generate the precision matrix $\Omega =
\Sigma^{-1}$ according to three different models: In the first model we
generate an Erdös--Rényi network  with edge probability 0.05. That is, the
adjacency matrix $A \in \reals^{p\times p}$ is generated as
\[
    A_{jk} = A_{kj} = \begin{cases}
        0 & \text{with probability } 0.95,\\
        1 & \text{with probability } 0.05. \\
    \end{cases}
\]
In the second model we generate a simple path network, with the adjacency
matrix $A$ having ones on the secondary diagonals. We denote this model by
\emph{AR(1)}, since it corresponds to a first order autoregressive model. In
the third model we generate a hub network as in \cite{Peng09}, where one node
has a high degree of connectivity (connected to 20\% of all other nodes), and
the rest of the nodes have between one and three connections. Such a model is
interesting in the context of gene regulatory networks, where a single gene may
regulate a large number of other genes.

In all three models we generate the precision matrix by randomly flipping the
signs of the adjacency matrix and setting the diagonal elements of matrix to
$\Omega_{jj} = 1 + \sum_{k\neq j}|A_{jk}|$ to guarantee positive definiteness
via strict diagonal dominance.

\paragraph*{\bf Regularization parameter selection.} For precision matrix and
partial correlation estimation we select the regularization parameters via
5-fold cross-validation. For sparsity pattern estimation, we select the
regularization parameter via maximizing the BIC. The set of values for
$\lambda$ considered is a geometrically spaced set between $10^{-2} \times
\max_{i\neq j} \widehat{S}_{ij}$ and $\max_{i\neq j} \widehat{S}_{ij}$, with
$\widehat{S}$ being the sample covariance matrix, taking inspiration from the
graphical lasso implementation in Scikit-learn \citep{sklearn}. For joint
partial regression we allow for a different regularization parameter
$\lambda_j$ for each feature, and select them separately via in the initial
regression step, as described in \S\ref{ssec:algorithm}.

\paragraph*{\bf Results.} In Figures~\ref{fig:precision-error} and
Figure~\ref{fig:partial-corr-error} we compare the Frobenius errors for
estimating the precision matrix and partial correlations, respectively. For
each network model and value of $p$, we fit the methods on 100 generated data
sets and used 5-fold CV with 20 values of $\lambda$ for model selection. The
results show that joint partial regression outperforms the benchmark methods
uniformly across network models and problem sizes in terms of estimation. 

In Figure~\ref{fig:f1-scores} we compare the $F_1$-scores for estimating the
network structure. For each network model and value of $p$, we fit the methods
on 20 generated data sets and used BIC with 100 values of $\lambda$ for model
selection. We see that no method uniformly outperforms the others. However,
joint partial regression is the only method that performs the best or second
best across all network models and problem sizes, while the others perform
poorly in atleast one scenario. Notably, joint partial regression has the best
relative performance on the hub network model. This may be explained by the
method's ability to have node specific regularization parameters. That is, we
are able to choose small values of $\lambda_j$ for hub nodes, and larger for
non-hub nodes, rather than putting the same level of regularization on all
nodes.

\begin{figure}
    \centering
    \begin{subfigure}[t]{0.6\linewidth}
        \centering
        \includegraphics[width=\linewidth]{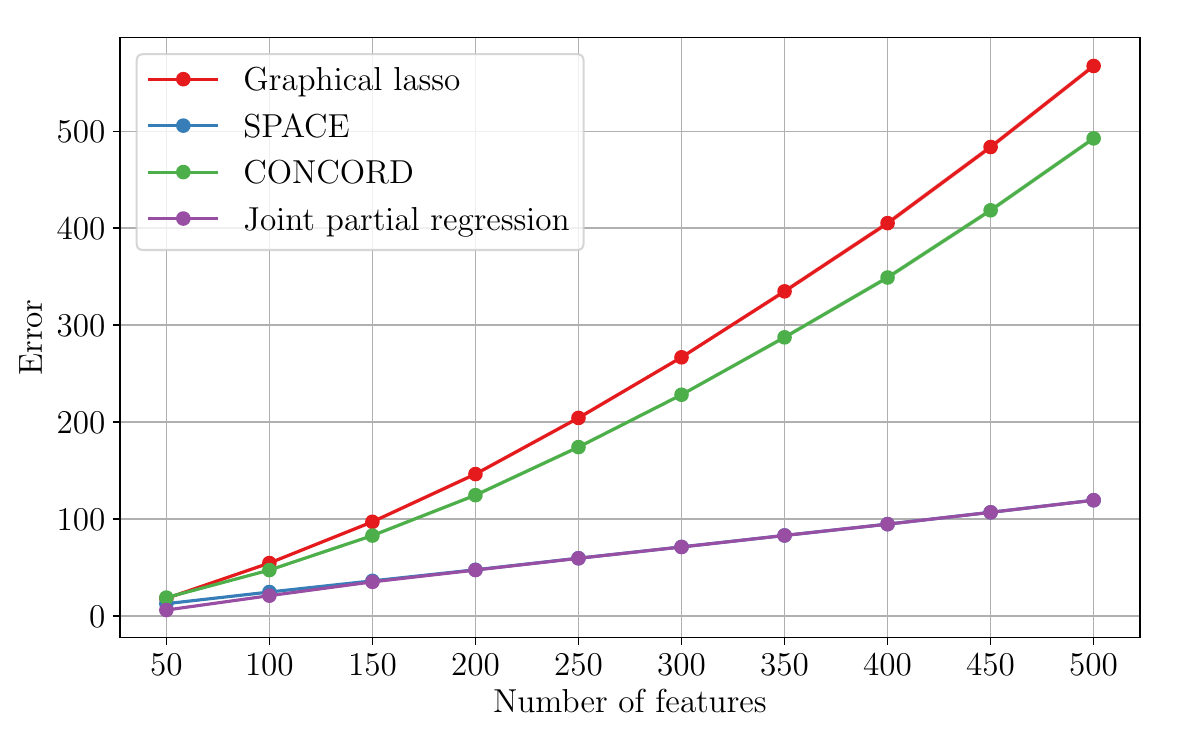}
        \caption{Erdös--Rényi model.}
        \label{fig:er_fro}
    \end{subfigure}

    \begin{subfigure}[t]{0.6\linewidth}
        \centering
        \includegraphics[width=\linewidth]{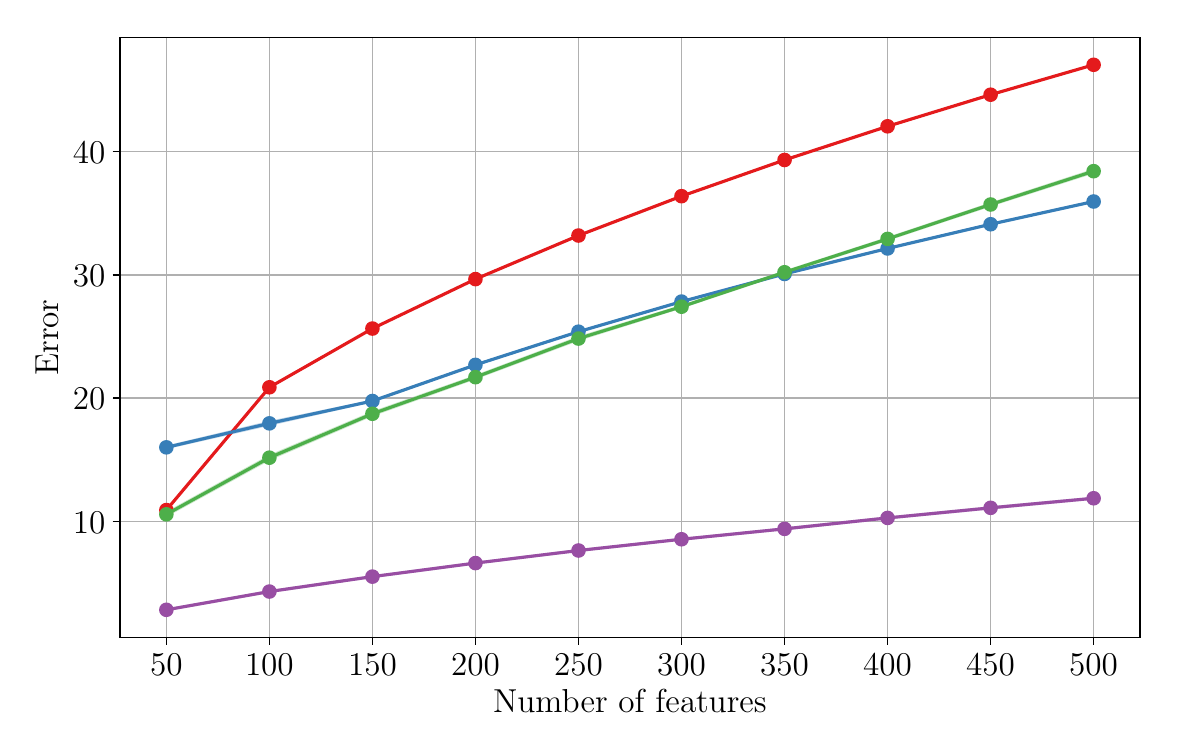}
        \caption{$\text{AR}(1)$ model.}
        \label{fig:ar_fro}
    \end{subfigure}

    \begin{subfigure}[t]{0.6\linewidth}
        \centering
        \includegraphics[width=\linewidth]{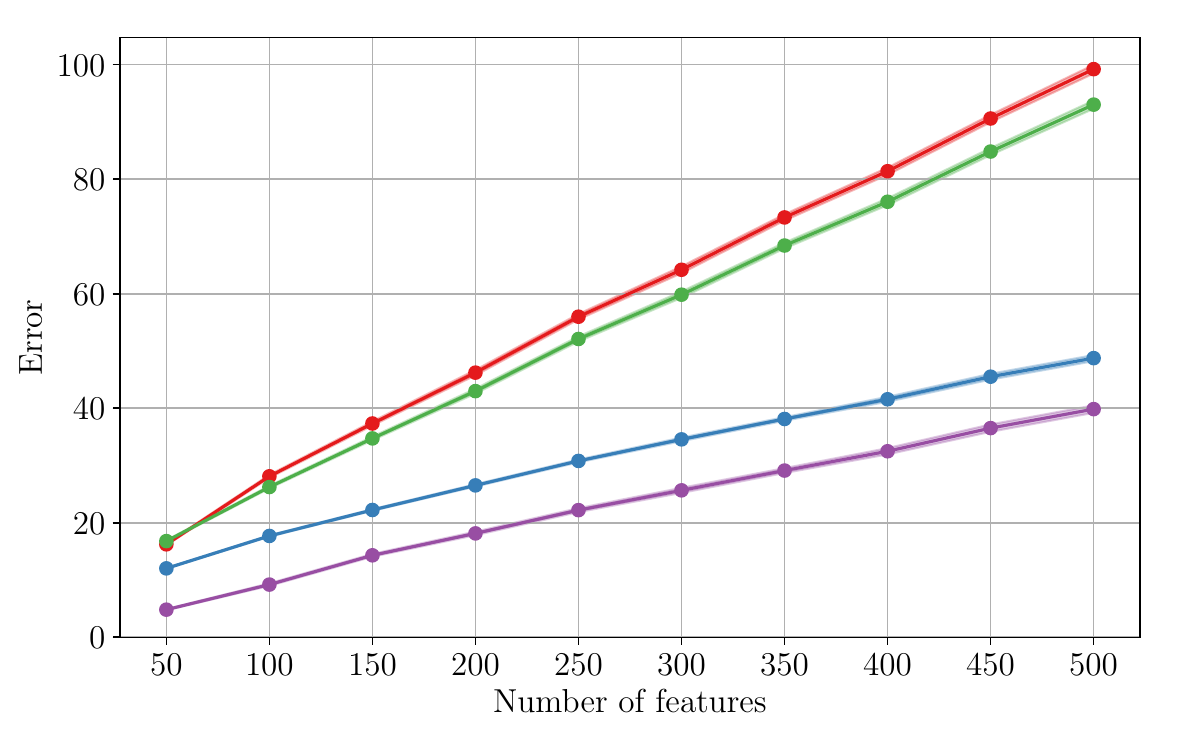}
        \caption{Hub network model.}
        \label{fig:hub_fro}
    \end{subfigure}
    \caption{Average Frobenius-error $\|\widehat{\Omega} -
    \Omega^\star\|_\mathrm{F}$ versus number of features with $\pm 2 \times
    \text{SE}$ bands.}
    \label{fig:precision-error}
\end{figure}

\begin{figure}
    \centering
    \begin{subfigure}[t]{0.6\linewidth}
        \centering
        \includegraphics[width=\linewidth]{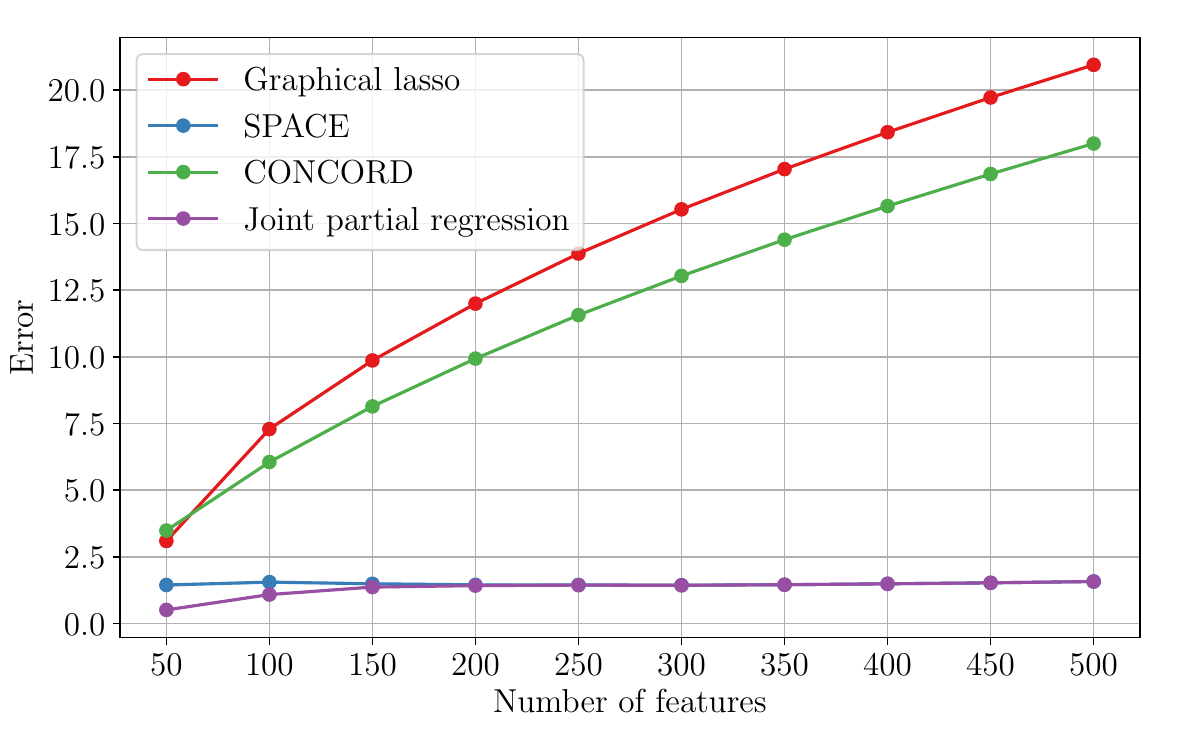}
        \caption{Erdös--Rényi model.}
        \label{fig:er_l2}
    \end{subfigure}

    \begin{subfigure}[t]{0.6\linewidth}
        \centering
        \includegraphics[width=\linewidth]{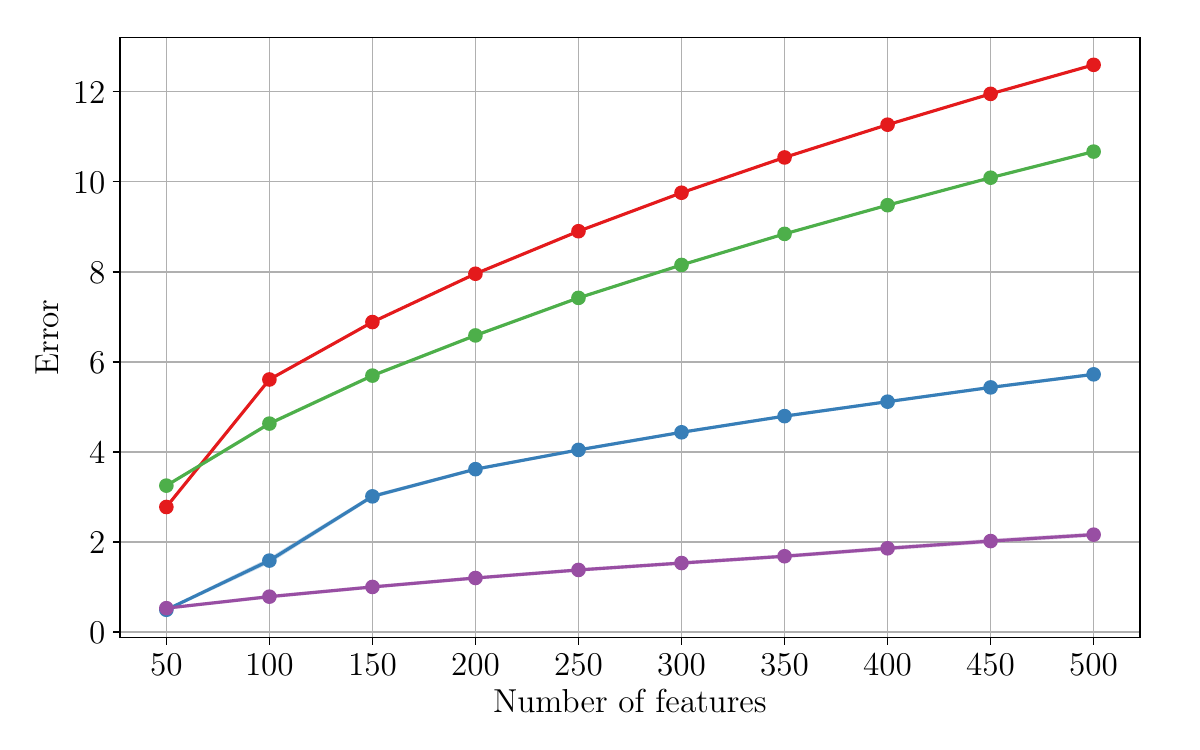}
        \caption{$\text{AR}(1)$ model.}
        \label{fig:ar_l2}
    \end{subfigure}

    \begin{subfigure}[t]{0.6\linewidth}
        \centering
        \includegraphics[width=\linewidth]{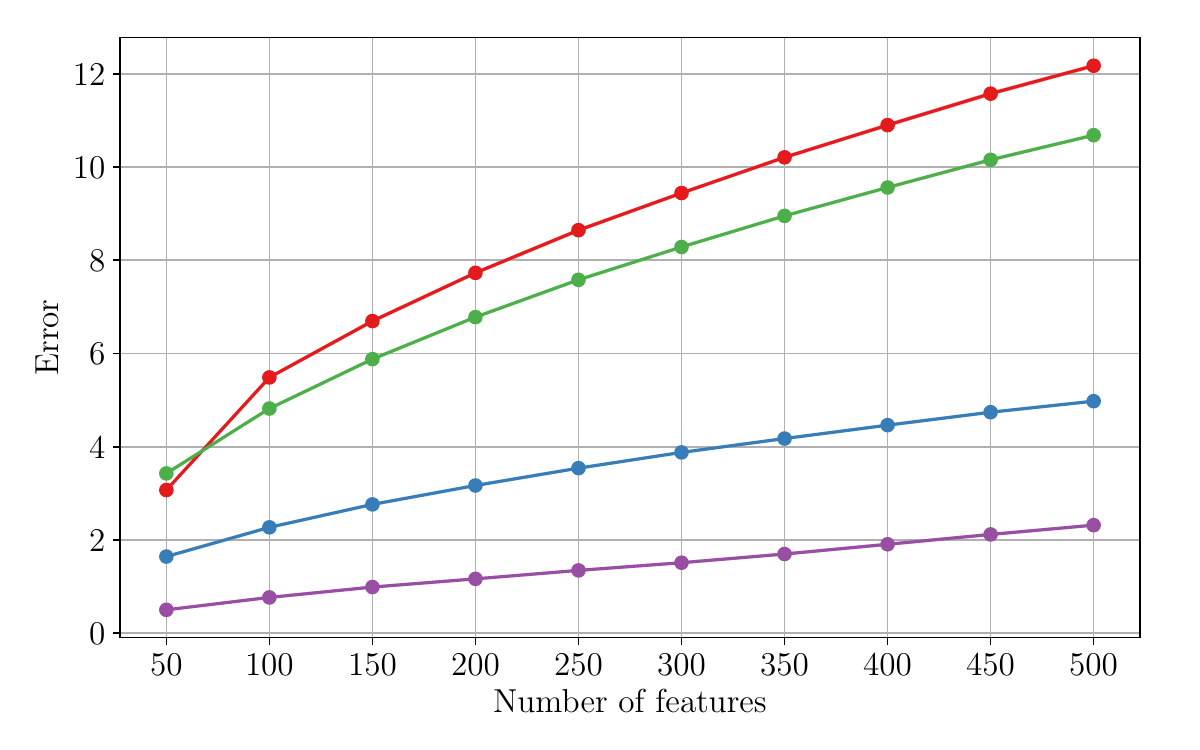}
        \caption{Hub network model.}
        \label{fig:hub_l2}
    \end{subfigure}
    \caption{Average Frobenius-error $\|\widehat{Q} - Q^\star\|_\mathrm{F}$
    versus number of features with $\pm 2 \times \text{SE}$ bands.}
    \label{fig:partial-corr-error}
\end{figure}

\begin{figure}
    \centering
    \begin{subfigure}[t]{0.6\linewidth}
        \centering
        \includegraphics[width=\linewidth]{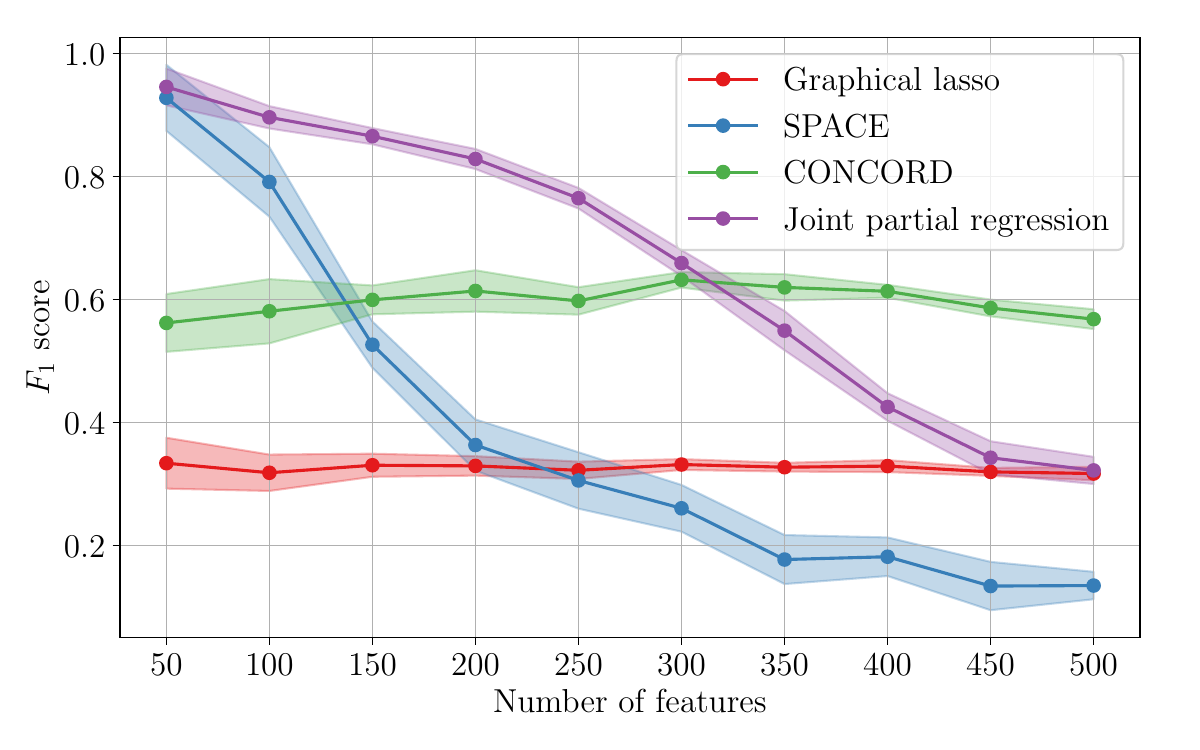}
        \caption{Erdös--Rényi model.}
        \label{fig:er_f1}
    \end{subfigure}

    \begin{subfigure}[t]{0.6\linewidth}
        \centering
        \includegraphics[width=\linewidth]{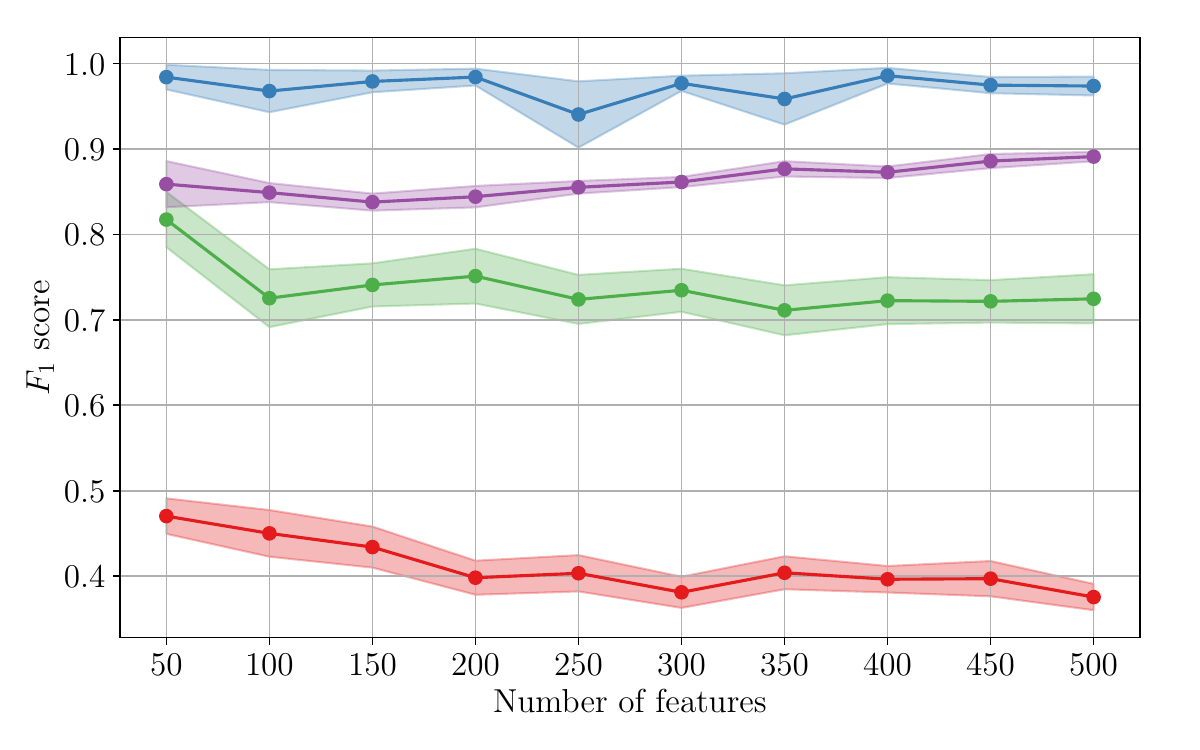}
        \caption{$\text{AR}(1)$ model.}
        \label{fig:ar_f1}
    \end{subfigure}

    \begin{subfigure}[t]{0.6\linewidth}
        \centering
        \includegraphics[width=\linewidth]{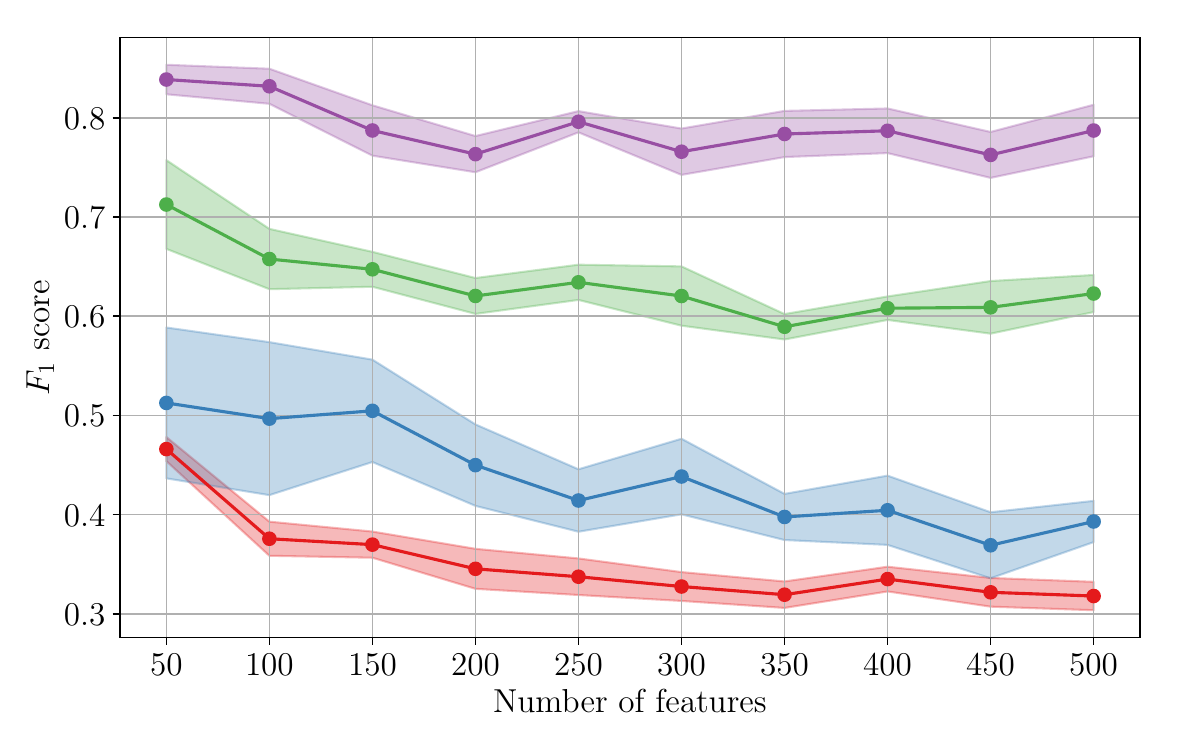}
        \caption{Hub network model.}
        \label{fig:hub_f1}
    \end{subfigure}
    \caption{Average $F_1$-scores versus number of features with $\pm 2 \times
    \text{SE}$ bands.}
    \label{fig:f1-scores}
\end{figure}

\subsection{Stock market data}\label{ssec:stock-market-data} 

To illustrate the performance of the proposed method on real data, we estimate
the partial correlation network of stocks on the Nasdaq Stock Market. 

\paragraph*{\bf Data set.} The data \citep{Nugent18} consists of daily prices of the
500 highest market capitalization stocks listed on stock exchanges in the
United States. The price data spans from the 11th of February 2013 to the 7th
of February 2018, totaling 1258 observations. We calculate the daily returns,
and after removing stocks with missing data and those without Nasdaq sector
data, we are left with 399 assets. 

\begin{table}[ht]
    \centering
    \caption{Sector to color mapping for the Nasdaq data set.}
    \begin{tabular}{llll}
    \toprule
    Sector & Color & Sector & Color \\
    \midrule
    Consumer Discretionary & \cellcolor[RGB]{238,114,0} & Technology         & \cellcolor[RGB]{79,112,148} \\
    Utilities              & \cellcolor[RGB]{66,125,60} & Finance            & \cellcolor[RGB]{0,178,227} \\
    Real Estate            & \cellcolor[RGB]{168,90,52} & Industrials        & \cellcolor[RGB]{205,108,242} \\
    Miscellaneous          & \cellcolor[RGB]{0,152,250} & Energy             & \cellcolor[RGB]{228,71,141} \\
    Health Care            & \cellcolor[RGB]{0,180,0}   & Consumer Staples   & \cellcolor[RGB]{144,139,0} \\
    Basic Materials        & \cellcolor[RGB]{0,163,128} & Telecommunications & \cellcolor[RGB]{255,58,66} \\
    \bottomrule
\end{tabular}
\label{tab:color}
\end{table}

\paragraph*{\bf Results.} The estimated networks are shown in
Figure~\ref{fig:asset-networks}. The nodes represent the stocks and color
indicates the sector to which the stock belongs (see Table~\ref{tab:color}).
Here, the networks have an approximate edge density of 1\%, which is the level
chosen by maximizing the BIC under the proposed method. Visually, the proposed
method produces the strongest sector-based separation, with the graphical lasso
showing similar but slightly weaker separation. In contrast, the networks of
\textsc{space} and \textsc{concord} have noticeably less separation.

\begin{figure}
    \centering
    \begin{subfigure}[t]{0.4\linewidth}
        \centering
        \includegraphics[width=\linewidth]{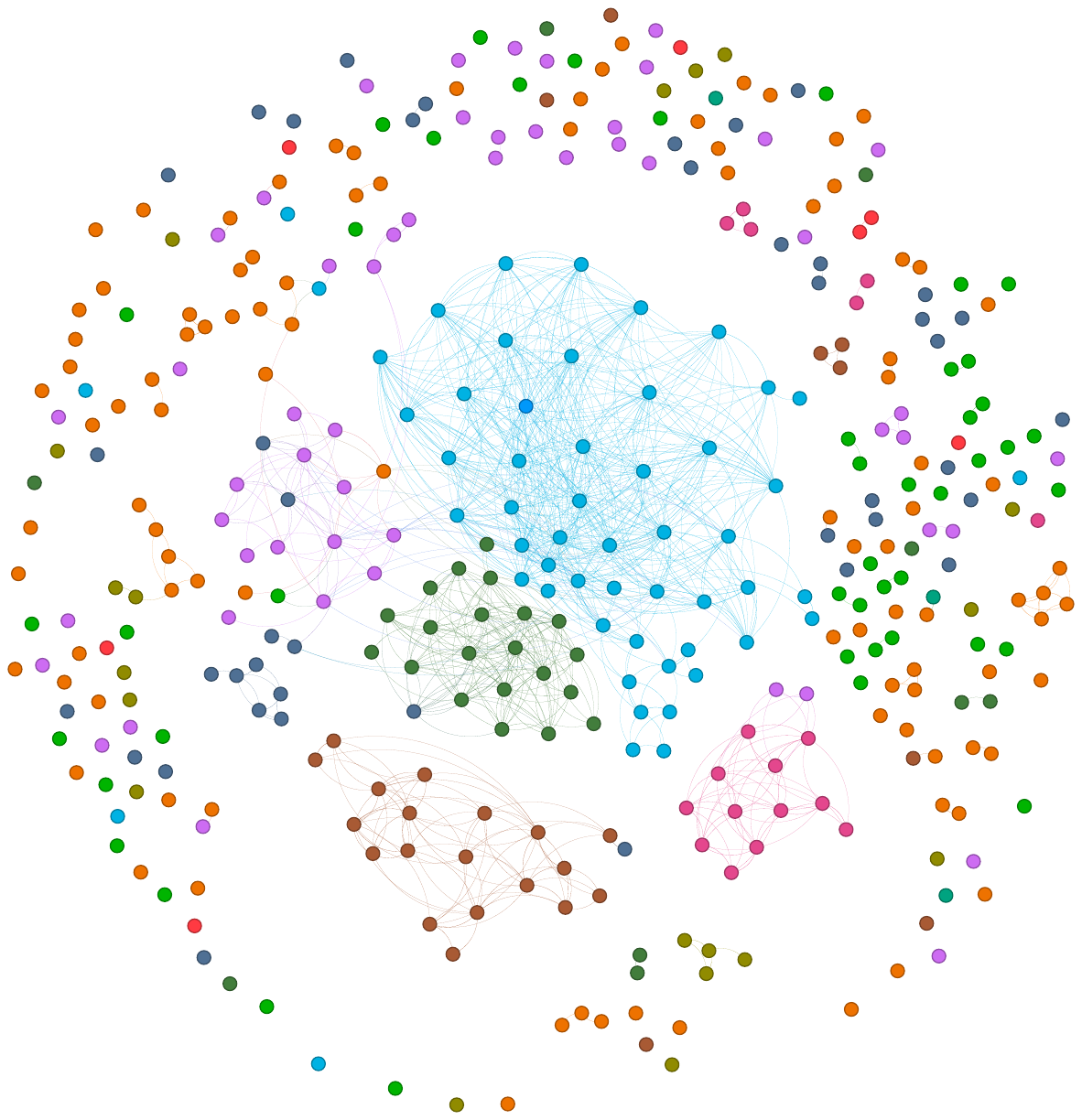}
        \caption{Graphical lasso}
        \label{fig:glasso-network}
    \end{subfigure}\hfill
    \begin{subfigure}[t]{0.4\linewidth}
        \centering
        \includegraphics[width=\linewidth]{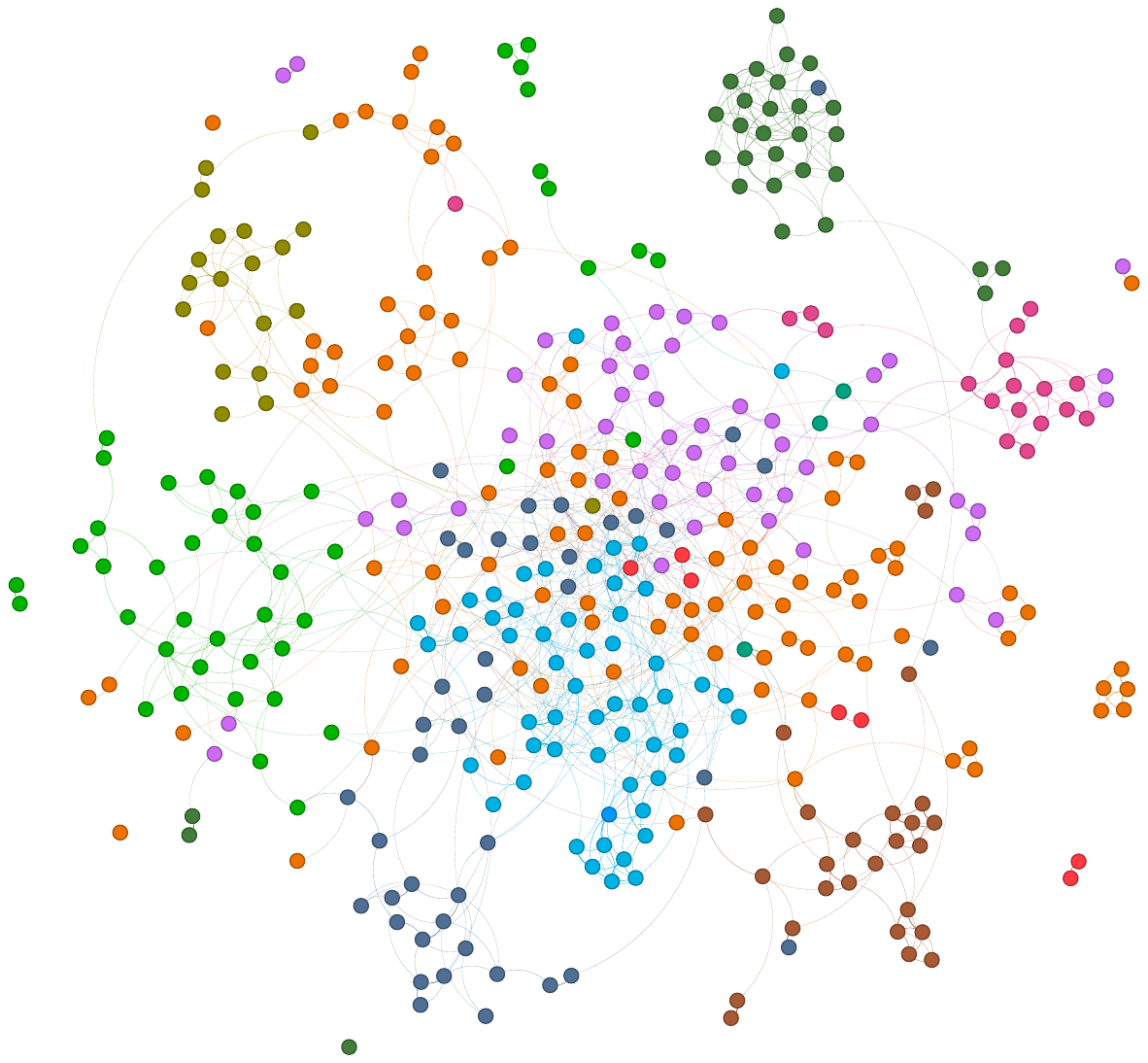}
        \caption{\textsc{space}} 
        \label{fig:space-network}
    \end{subfigure}
    \begin{subfigure}[t]{0.4\linewidth}
        \centering
        \includegraphics[width=\linewidth]{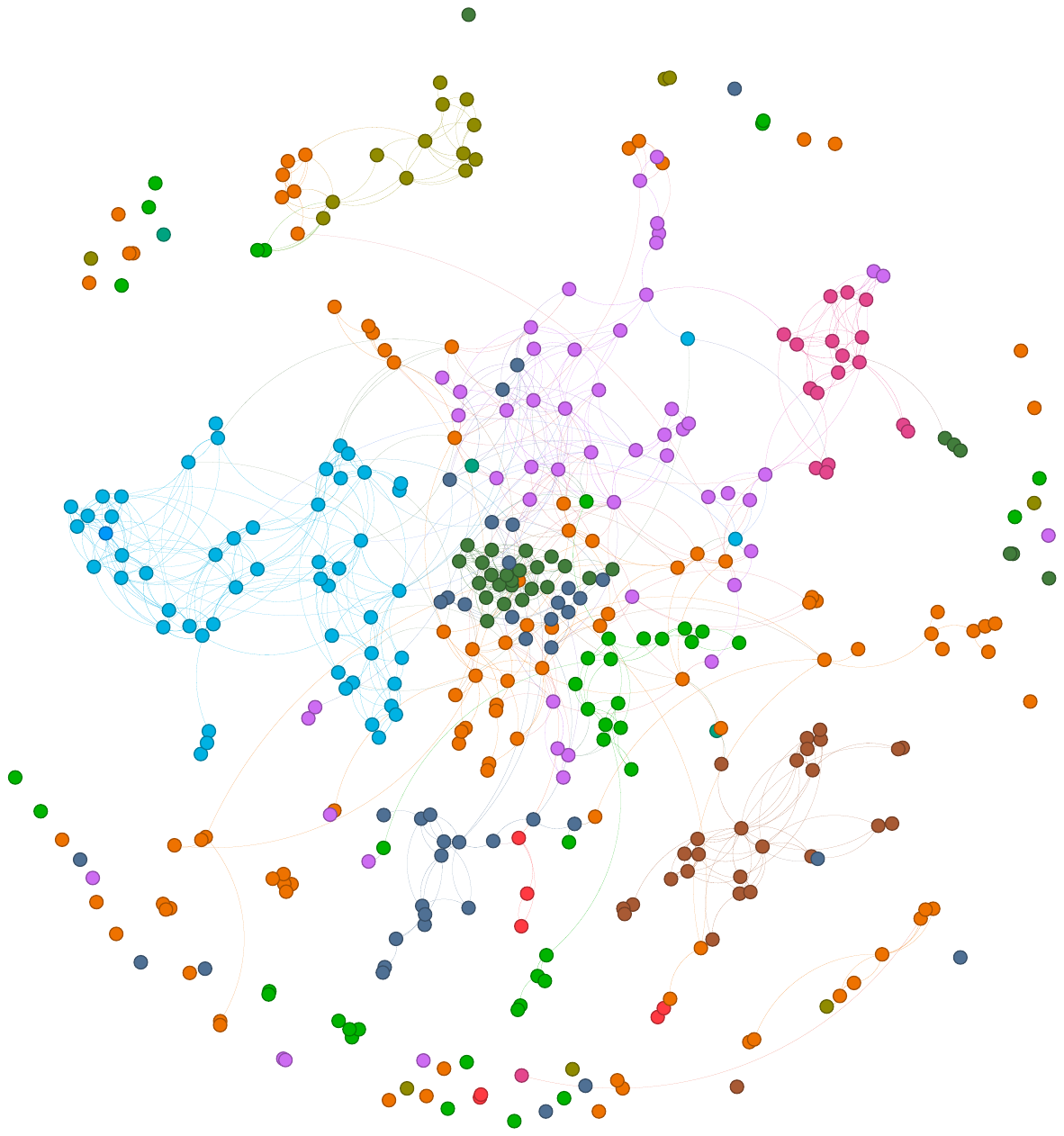}
        \caption{\textsc{concord}} 
        \label{fig:concord-network}
    \end{subfigure}\hfill
    \begin{subfigure}[t]{0.4\linewidth}
        \centering
        \includegraphics[width=\linewidth]{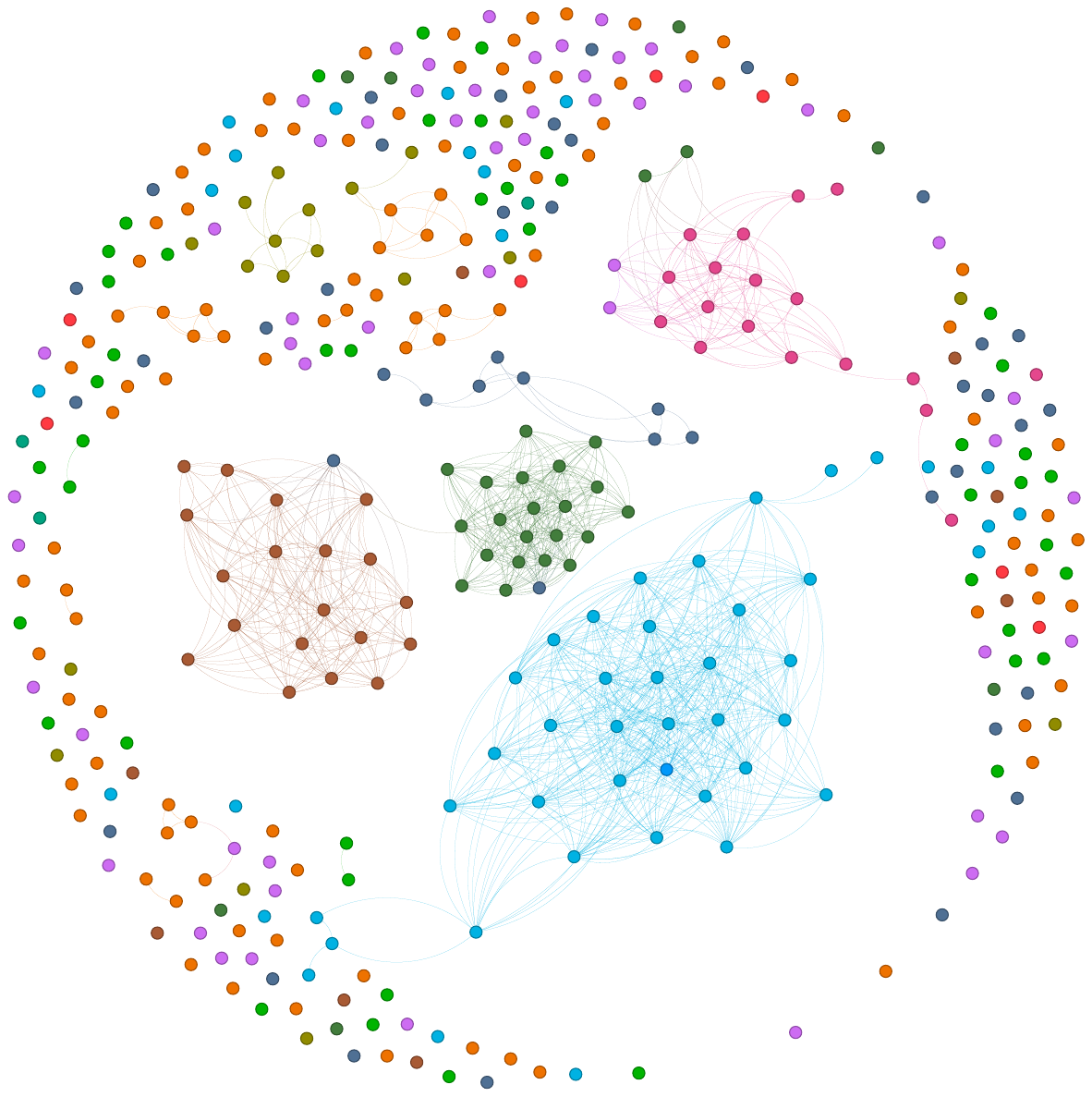}
        \caption{Joint partial regression} 
        \label{fig:jpr-network}
    \end{subfigure}
    \caption{Estimated stock market network at an approximate edge density of 1\%.}
    \label{fig:asset-networks}
\end{figure}

\section{Discussion}

We propose a method for estimating high-dimensional sparse precision and
partial correlation matrices based on jointly regressing each feature on the
remaining features. Our theoretical analysis of partial correlation estimation
provides what is, to the best of our knowledge, the first non-asymptotic rate
on partial correlation estimation. We show that the method matches the rate of
the graphical lasso for precision matrix estimation, and improves upon the
asymptotic rate of \textsc{space} by more than a factor of $\sqrt{s}$ for
partial correlation estimation. We derive an efficient algorithm for solving
the joint partial regression problem based on the \textsc{pd3o} algorithm. The
method is shown to perform well on both synthetic and real data, producing
interpretable results and outperforming the benchmarks in terms of estimation
error across network types and problem sizes.

\section{Disclosure statement}\label{disclosure-statement}

The authors have no conflicts of interest to declare.


\section{Data Availability Statement}\label{data-availability-statement}

Deidentified data is available at the following URL:
\url{https://www.kaggle.com/datasets/camnugent/sandp500}.

\phantomsection\label{supplementary-material}
\bigskip

\appendix
\begin{center}

	{\large\bf SUPPLEMENTARY MATERIAL}

\end{center}

\begin{description}
    \item[Proof of theoretical results:] 
        Complete proof of Theorem~\ref{thm:est-rate}. (PDF)
    \item[Python package for Joint Partial Regression:]
        Python package \texttt{graph-jpr} containing implementation of the
        proposed method with subroutines written in Rust. The package also
        contains code necessary to reproduce experiments. (GNU zipped tar file)
\end{description}

\bibliography{references}
\end{document}